\documentclass[journal,letterpaper,twocolumn,twoside,nofonttune]{IEEEtran}
\usepackage{mathpazo}
\usepackage{times}
\usepackage[T1]{fontenc}
\usepackage{amsmath,amssymb,amsfonts}
\usepackage{mathrsfs}
\usepackage{mathabx}
\usepackage{amsbsy}
\usepackage{graphicx}
\usepackage{graphics}
\usepackage{epstopdf}
\usepackage{algorithm}
\usepackage{algorithmic}
\usepackage{subfigure}
\usepackage{cite}
\usepackage{multirow}
\usepackage{ctable}
\usepackage{caption}
\usepackage{setspace}
\usepackage{mathtools}
\usepackage{hyperref}
\usepackage[algo2e]{algorithm2e}

\newtheorem{theorem}{{Theorem}}
\newtheorem{lemma}[theorem]{{Lemma}}

\newcommand{\cF}{{\cal F}}

\newcommand{\cN}{{\cal N}}

\DeclareMathAlphabet{\mathbfsl}{OT1}{ppl}{b}{it} 

\newcommand{\bA}{\mathbfsl{A}} 
\newcommand{\bB}{\mathbfsl{B}}

\newcommand{\bI}{\mathbfsl{I}}

\newcommand{\bN}{\mathbfsl{N}}

\newcommand{\bX}{\mathbfsl{X}}
\newcommand{\bY}{\mathbfsl{Y}}

\newcommand{\ba}{\mathbfsl{a}} 
 
\newcommand{\bu}{\mathbfsl{u}} 

\newcommand{\bw}{\mathbfsl{w}} 
\newcommand{\bx}{\mathbfsl{x}}
\newcommand{\by}{\mathbfsl{y}} 
\newcommand{\bz}{\mathbfsl{z}}
\newcommand{\bl}{\mathbfsl{l}}

\newcommand{\ceil}[1]{\left\lceil #1 \right\rceil}

\def\QEDclosed{\mbox{\rule[0pt]{1.3ex}{1.3ex}}} 

\newcommand*{\rom}[1]{\expandafter\romannumeral #1}

\def\QED{\QEDclosed} 

\makeatletter
\newcommand{\AlignFootnote}[1]{%
  \ifmeasuring@
  \else
    \iffirstchoice@
      \footnote{#1}%
    \fi
  \fi}
\makeatother

\newcommand{\be}{\begin{equation}}
\newcommand{\ee}{\end{equation}} 
\newcommand{\eq}[1]{(\ref{#1})}

\renewcommand{\leq}{\leqslant}
 
\renewcommand{\geq}{\geqslant}

\renewcommand{\Bbb}{\mathbb}
\newcommand{\C}{{\Bbb C}} 

\newcommand{\R}{{\Bbb R}} 

\newcommand{\D}{{\Bbb D}}
\newcommand{\F}{{\Bbb F}}

\newcommand{\Tref}[1]{Theo\-rem\,\ref{#1}}

\newcommand{\Lref}[1]{Lem\-ma\,\ref{#1}}
\newcommand{\Cref}[1]{Co\-ro\-lla\-ry\,\ref{#1}}

\newcommand{\deff}{\mbox{$\stackrel{\rm def}{=}$}}

\newcommand{\norm}[1]{\left\lVert#1\right\rVert}
\usepackage[utf8]{inputenc}

\begin{document}

\title{\Huge$\,$\\[-2.75ex]
{Privacy-Preserving Distributed Learning\\ in the Analog Domain}\\[0.50ex]}

\author{\large%
Mahdi Soleymani, Hessam Mahdavifar, and A. Salman Avestimehr
\vspace{-.25in}
\thanks{This work was supported by the National Science Foundation
under grants CCF--1763348, CCF--1909771, CCF--1941633.}
\thanks{M.\ Soleymani and H.\ Mahdavifar are with the Department of Electrical Engineering and Computer Science, University of Michigan, Ann Arbor, MI 48104 (email: mahdy@umich.edu and hessam@umich.edu).}
\thanks{A.\ Salman Avestimehr is with the Department of Electrical
Engineering, University of Southern California, Los Angeles, CA 90089 USA
(e-mail: avestimehr@ee.usc.edu).}
}


\maketitle

\begin{abstract}
We consider the critical problem of distributed learning over data while keeping it private from the computational servers. 
The state-of-the-art approaches to this problem rely on quantizing the data into a finite field, so that the cryptographic approaches for secure multiparty computing can then be employed. These approaches, however, can result in substantial accuracy losses due to fixed-point representation of the data and computation overflows.
To address these critical issues, we propose a novel algorithm to solve the problem when data is in the analog domain, e.g., the field of real/complex numbers. We characterize the privacy of the data from both information-theoretic and cryptographic perspectives, while establishing a connection between the two notions in the analog domain. More specifically, the well-known connection between the distinguishing security (DS) and the mutual information security (MIS) metrics is extended from the discrete domain to the continues domain. This is then utilized to bound the amount of information about the data leaked to the servers in our protocol, in terms of the DS metric, using well-known results on the capacity of single-input multiple-output (SIMO) channel with correlated noise. It is shown how the proposed framework can be adopted to do computation tasks when data is represented using floating-point numbers. We then show that this leads to a fundamental trade-off between the privacy level of data and accuracy of the result.
As an application, we also show how to train a machine learning model while keeping the data as well as the trained model private. Then numerical results are shown for experiments on the MNIST dataset. Furthermore, experimental advantages are shown comparing to fixed-point implementations over finite fields.
\end{abstract}

\section{Introduction}

It is estimated that $2.5 \times 10^{18}$ bytes of data are generated every day with a pace that is only accelerating as $90$ percent of the data in the world has been generated in the past two years. Datasets with massive size need to be processed at an unprecedented scale, which makes it imperative to provide scalable solutions for large computational jobs associated with learning problems to be performed in a distributed fashion \cite{abadi2016tensorflow}. In such distributed systems, data is dispersed among many servers that operate in parallel with the aim of collectively completing a certain computational job, e.g., computing a certain function over the dataset. Then the results generated by \textit{sufficiently} many local servers are collected in order to recover the desired outcome, e.g., the output of the given function over the dataset. 

One of the major concerns in such distributed learning systems is to preserve the privacy of the dataset while dispersing it among the servers. More specifically, the dataset may contain highly sensitive information, e.g., biometric data of patients in a hospital \cite{raghupathi2014big} or customers' data of a company \cite{mcafee2012big}, necessitating that \textit{almost} no information about the dataset is revealed to the computational servers. Such a privacy constraint is often generalized to ensure that any subset of colluding servers, up to a certain size, can not gain \textit{almost} any information about the dataset. 

The privacy of data can be measured in terms of various metrics, including information-theoretic security \cite{shannon1949communication} as well as well-known notions of semantic security and distinguishing security in the cryptography literature emanating from \cite{goldwasser1984probabilistic}. Fundamental connections between these notions are established in \cite{bellare2012cryptographic}. The seminal Shamir's secret sharing scheme and its various versions are often used to provide information-theoretic security for data, referred to as a secret, while distributing it among a set of servers/users \cite{shamir1979share}. Also, Shamir's scheme serves as the backbone of most of the existing schemes on privacy-preserving distributed computing such as the celebrated BGW scheme \cite{ben1988completeness}. The idea can be illustrated via an example as follows. Consider a given dataset $X$ and two computational servers, referred to as servers $1$ and $2$. Suppose that the function $f(X)=aX$, where $a$ is a scaler, needs to be computed over the dataset $X$. The data symbols in $X$ as well as $a$ are considered as elements of a finite field $\F_q$. Then a random $N$ is generated, with the same size as the dataset $X$ and entries generated independently and uniformly at random from $\F_q$. Then $N$ and $X+N$, also referred to as secret shares, are given to the servers $1$ and $2$, respectively. Since $N$ and $X+N$ are both uniformly distributed, the servers do not learn anything about $X$ individually. The servers return $aN$ and $a(X+N)$. Then $aX$ is recovered by subtracting the former from the latter.


In Shamir's scheme, the secret/data symbols are always assumed to be elements of a finite field. Consequently, the state-of-the-art schemes treat the data symbols in the given dataset as finite field elements in order to employ Shamir's secret sharing, see, e.g., \cite{ben1988completeness}. However, quantizing the data into a finite field can result in substantial accuracy losses mainly due to computation overflows. In practice, the dataset $X$ consists of real/complex values often represented as floating-point numbers. Then $X$ can not be perfectly secured in an information-theoretic sense, i.e., the mutual information between $X$ and $X+N$, denoted by $I(X;X+N)$, being exactly zero. These are the main challenges that need to be properly addressed when designing privacy-preserving distributed learning algorithms in the infinite fields of $\R / \C$, also referred to as the \textit{analog domain}.


\subsection{Our contributions}
\label{sec:one-A}

In this paper, we provide a framework to construct the counterpart of Shamir's secret sharing scheme in the analog domain. This framework is then used to construct privacy-preserving distributed computation and learning protocols over real/complex datasets. In other words, all the operations including encoding the data symbols to be distributed among the computational servers and recovery of the final outcome from the collected results returned by the servers are over the infinite fields of $\R / \C$. It is assumed that the servers are \textit{honest-but-curious} meaning that they will not deviate from the protocol but may attempt to infer the data from what they observe throughout the protocol.

In the proposed protocol, the information-theoretic measure of security is no longer \textit{perfect}, comparing to Shamir's secret sharing scheme over finite fields, as discussed earlier. In order to show the privacy guarantees of the protocol, bounds are provided on how much information about data is revealed to a server/subsets of servers in terms of various notions of security. We also argue that, in a practical setting, this comes at the expense of accuracy of the final outcome of the protocol when all data symbols are represented by floating-point numbers and all operations are also assumed to follow \textit{standard} floating-point operations. More precisely, we provide a fundamental trade-off between the security level of the protocol and the accuracy of the outcome in a practical setting assuming floating-point operations. The proposed protocol is also used in a distributed learning experiment using four servers to train a logistic regression model over MNIST dataset \cite{lecun2010mnist}. In this experiment, the amount of information about the dataset and the trained model revealed to each of the servers, in terms of the distinguishing security metric, is less than $10^{-15}$ and $2\times 10^{-14}$, respectively.  It is observed that the accuracy of our protocol closely follows that of the conventional centralized approach, thereby offering a privacy-preserving distributed solution at a negligible cost in terms of the accuracy of the result. Furthermore, it is shown that while approaches based on fixed-point implementations suffer from a sharp transition to the performance of randomly guessing by increasing the size of training dataset, our protocol offers a robust solution that is scalable with the size of the training dataset.



\subsection{Related work}

Studying privacy-preserving distributed machine learning algorithms has recently received significant attention in the literature \cite{wagh2018securenn,yu2018lagrange,dahl2018private,barak2019secure,so2019codedprivateml,kumar2019cryptflow}. There is also an extensive amount of work on secure matrix-matrix multiplication which is a core building block for many machine learning algorithms, see, e.g., \cite{yu2020entangled,aliasgari2020private,d2020gasp,bitar2019private,nodehi2019secure}. As mentioned earlier, the information-theoretic privacy guarantees of the data in these prior works is based upon Shamir's secret sharing scheme and its variations. Since Shamir's scheme needs to be run over a finite field while data symbols are real-valued, a common method is to assume a certain quantization of the data symbols followed by mapping them into elements of a finite field of a large prime size. However, if an overflow occurs, i.e., a computed symbol during the computation process by one of the servers becomes larger than the field size, then a successful recovery of the outcome of the computation can not be guaranteed. In other words, the computation procedure at each of the servers can be regarded as a fixed point computation, which is constrained by conditions guaranteeing no overflow occurs. 

There is also another line of work on adopting coded distributed computing protocols for computation over real-valued data \cite{fahim2019numerically,ramamoorthy2019numerically,das2019distributed,charalambides2020numerically}. But these works are mostly focused on the numerical stability of the protocols in the presence of slow or unresponsive servers, also referred to as stragglers, and do not study privacy guarantees for the data. Our focus in this paper is on providing privacy-preserving protocols and straggler servers are not considered. Also, codes in the analog domain have been recently studied in the context of block codes \cite{roth2020analog} as well as subspace codes \cite{soleymani2019analog} for analog error correction. However, secret sharing and privacy-preserving computation in the analog domain are not discussed in these works.

Another related major line of work concerns with floating-point implementation of secure multi-party computing (MPC) protocols \cite{setty2012taking,aliasgari2013secure,catrinatowards}. Such protocols can be described in high level as follows. In a \textit{standard} floating-point implementation, each number/data symbol is represented by two main components: one represents the most, let's say $v$, significant bits of the data symbol and the other one represents the power of the exponent of data symbol. Then these two components of the data symbols are secured separately using Shamir's secret sharing over finite fields. This requires a certain implementation of floating-point operations and does not allow using off-the-shelf readily available floating-point operations that are often optimized to perform computational tasks on badges of data in parallel. As a result, the inefficiency of such protocols poses a major difficulty in their implementation. Furthermore, a major difference between this line of work and our approach is that the parties are allowed to communicate in secure multiparty computing. This is mainly because in this setting each party aims at computing a certain function of data symbols shared between the parties without revealing any information about his/her share of the data. The communication overhead between the parties/servers is another major factor contributing to the inefficiency of these protocols in practical systems. On the other hand, in our approach, different servers are assumed to run in parallel and no communication is required between them. Once finished, the servers return their locally computed results from which the true outcome of the computational task can be computed.

The rest of this paper is organized as follows.
The problem is formulated in Section\,\ref{sec::system_model} followed by the description of the proposed protocol. The accuracy of the the protocol is analyzed in Section\,\ref{sec::accuracy}. In Section\,\ref{sec:privacy} we provide an analysis for the privacy level of data in the protocol by considering two well-known notions of security. Various experimental results are provided in Section\,\ref{sec:experiments}. Finally, the paper is concluded in Section\,\ref{sec:conclusion}.


\section{Problem Formulation and the Protocol}\label{sec::system_model}


Consider a setup with $N$ computation servers/parties indexed by $1,2,\dots,N$. Given a data symbol $s$, also referred to as a \emph{secret}, a $D$-degree polynomial function of $s$ denoted by $f(s)$ needs to be computed by utilizing the computational power of the parties, while the secret remains \textit{private} assuming up to $t$ parties can collude. The notion of privacy will be clarified in Section\,\ref{sec:privacy}. The secret $s$ is an instance of a continuous random variable $S$ taking values in $ [-r,r] $. \footnote{Following the convention, random variables are represented by capital letters and their instances are represented by lower case letters.} Other than this constraint on the range of $S$, no assumption is made on the probability distribution of $S$. 

\noindent
\textbf{Remark 1.} Note that the computational task of polynomial evaluation is considered in this paper in order to arrive at explicit analytical guarantees. However, in order to apply this setup to a learning experiment, a polynomial approximation of the underlying computation function, e.g., the sigmoid function, can be considered. This will be further discussed in Section\,\ref{sec:experiments}. 

In the considered protocol, given the secret $s$ the polynomial $p(x)$ is constructed as follows:
$$
p(x)\,\deff\, s+\sum_{j=1}^{t}n_jx^j,
$$
where $n_j$'s are i.i.d., drawn from a zero-mean circular symmetric complex Gaussian distribution with standard deviation $\frac{\sigma_n}{\sqrt{t}}$, denoted by $\cN(0,\frac{\sigma_n^2}{t})$, where $t$ is the maximum number of colluding parties. For evaluating the precision of the protocol in practice, the distribution of $n_i$'s will be truncated, i.e., it is assumed that they are drawn from a truncated Gaussian distribution with a maximum absolute value, denoted by $m$, for $m\in \R^+$. This will be further clarified in Section\,\ref{sec::accuracy}. The shares of the computation parties consist of the evaluation of $p(x)$ over certain complex-valued evaluation points $\omega_1,\hdots,\omega_t$, i.e., 
\begin{align}
\label{system}
y_i=s+\sum_{j=1}^{t}\omega_i^jn_j
\end{align}
is given to server $i$, for $i \in [N]$. In the next section, it is shown how to pick $\omega_i$'s in order to maximize the accuracy. The system of equations in \eq{system} can be written in the matrix form as follows:
\begin{equation}
\label{A-def}
\by_{N \times 1} = \bA_{N \times (t+1)} \bx_{(t+1) \times 1},
\end{equation}
where $\bx = (s,n_1,\hdots,n_t)^{\text{T}}$, $\by=(y_1,\hdots,y_N)^{\text{T}}$, and 
\[
\bA \ \deff\
\begin{bmatrix}
    1 & \omega_1 & \dots &\omega_1^{t} \\
    1 & \omega_2 & \dots &\omega_2^{t} \\
    \vdots & \vdots & \ddots & \vdots \\
    1 & \omega_{N} & \dots &\omega_{N}^{t} 
\end{bmatrix}.
\]
Then server $i$ computes $f(y_i)$ and returns the result, e.g., to a \textit{master} node. The master node then recovers $f(s)$. Conceptually, this can be done by interpolating the polynomial $f\bigl(p(x)\bigr)$ and evaluating it at $0$, i.e., the constant coefficient of $f\bigl(p(x)\bigr)$ is equal to $f(s)$. More specifically, let $\ba = (f(s),a_1, \hdots, a_{d})^{\text{T}}$, where $d = Dt$, denote the vector of all coefficients of the polynomial $f(p(x))$. Let also $\bz = (f(y_1), \hdots, f(y_N))^{\text{T}}$ and 
\[
\bB \ \deff\
\begin{bmatrix}
    1 & \omega_1 & \dots &\omega_1^{d} \\
    1 & \omega_2 & \dots &\omega_2^{d} \\
    \vdots & \vdots & \ddots & \vdots \\
    1 & \omega_{N} & \dots &\omega_{N}^{d} 
\end{bmatrix}.
\]
Then the system of linear equations 
\begin{equation}
\label{a-sol}
\bz_{N \times 1}=\bB_{N\times (d+1)} \ba_{1\times (d+1)},
\end{equation}
can be solved for $\ba$ in order to recover $f(s)$. Note that $N \geq Dt+1 = d+1$ is the necessary and sufficient condition on the number of parties in order to guarantee a successful interpolation of $f\bigl(p(x)\bigr)$, which is of degree $d$. Equivalently, it is the necessary and sufficient condition for recovery of $\ba$ in \eq{a-sol}. Throughout the rest of this paper, it is assumed that $N=d+1$, implying that all shares $y_i$'s are needed to be returned to the master node for a successful recovery of the computation

Note that the master node does not need to compute the entire $\ba$ in \eq{a-sol} and is only interested in recovering $f(s)$, the first entry of $\ba$. Let $\tilde{\boldsymbol{b}}$ denote the first row of $\bB^{-1}$, which is well-defined due to $\bB$ being a Vandermonde matrix. Then the master node only needs to compute $\tilde{\boldsymbol{b}}\bz$ to recover $f(s)$. Since $\omega_i$'s are fixed, $\tilde{\boldsymbol{b}}$ is computed once, is stored, and then is used every time the protocol is run.\\
\noindent
\textbf{Remark 2.}
Note that the computation complexity of encoding in the master node is linear with the dataset size, where the dataset is treated as a vector of secrets and $f(.)$ needs to be evaluated over the entries of this vector. Moreover, the complexity of decoding is also linear with the dataset size as the decoder only computes a linear combination of the results returned by the servers. In other words, the computation complexity at the master node does not depend on $D$, which can be large. It is worth mentioning that the goal of the protocol is not to reduce the \textit{overall} computation complexity of a computation task across all the servers. The protocol in this paper, as well as prior works in the literature, e.g., \cite{so2019codedprivateml}, provide a framework to utilize external computation units in distributed servers while providing privacy guarantees.

To summarize, the protocol is described step-by-step in Algorithm\,\ref{scheme} next.

\begin{algorithm}[H]
\SetAlgoLined
 \textbf{Input:} Secret $s$.\\
 \textbf{Public parameters:} $\bA_{N \times (t+1)}$, $\tilde{\mathbf{b}}_{1 \times N}$.   \\
\textbf{Output:} Evaluation of $f(s)$ in the master node.

\textbf{Encoding phase (at the master):}

Pick i.i.d. $n_j \sim \cN(0,\frac{\sigma_n^2}{t})$, for $j=1,\dots,t$.

Set $\bx = (s,n_1,\hdots,n_t)^{\text{T}}$.

Compute $(y_1,\hdots,y_N)^{\text{T}}=\bA\bx$.

Send $y_i$ to server $i$.
  
\textbf{Computation phase (at server $i$):}
  
Compute $z_i= f(y_i)$.

Send $z_i$ the to the master node.
  
  \textbf{Decoding phase (at the master):}

Set $\bz = (z_1, \dots, z_N)^{\text{T}}$.

Compute $f(s)=\tilde{\mathbfsl{b}}\bz$.

\caption{Privacy-preserving distributed polynomial evaluation scheme in the analog domain. }
\label{scheme}
\end{algorithm}

In the next section, the accuracy of the protocol described in Algorithm\,1 is analyzed. In theory, if all the computations are done over the complex numbers with infinite precision, then $f(s)$ is computed accurately. In practice, data is represented using a finite number of bits, either as fixed point or floating point. Floating-point representation consists of a fixed-precision part and an exponent part specifying how the fixed-precision part is scaled. Let $v$ denote the number of precision bits in the floating-point representation, i.e., the $v$ most significant bits are kept in the fixed-precision part. Let also $q$ denote the number of bits used to represent the power of the exponent part in the floating-point representation. 

\section{Accuracy analysis}\label{sec::accuracy}


In this section, accuracy of the computation outcome of the proposed protocol in Section\,\ref{sec::system_model} is characterized in terms of other parameters of the protocol. Furthermore, it is shown how to pick the evaluation points in the protocol in order to maximize the accuracy. 

In general, in a system of linear equations $\bA\bx=\boldsymbol{y}$, where $\bx$ is the vector of unknown variables, the perturbation in the solution caused by the perturbation in $\boldsymbol{y}$ is characterized as follows. Let $\hat{\boldsymbol{y}}$ denote a noisy version of $\boldsymbol{y}$, where the noise can be caused by round-off errors,  truncation, etc. Let also $\hat{\bx}$ denote the solution to the considered linear system when $\boldsymbol{y}$ is replaced by $\hat{\boldsymbol{y}}$. Let $\Delta \bx\deff \hat{\bx}-\bx$ and $\Delta \boldsymbol{y}\deff \hat{\boldsymbol{y}}-\boldsymbol{y}$ denote the perturbation, also referred to as error, in $\boldsymbol{x}$ and $\boldsymbol{y}$, respectively. Then the relative perturbations of $\boldsymbol{x}$ is bounded in terms of that of $\boldsymbol{y}$ as follows  \cite{demmel1997applied}:
\be
\label{relative_error_ineq}
\frac{\norm{\Delta \bx}}{\norm{\bx}}\leq \kappa_{\bA}\frac{\norm{\Delta \boldsymbol{y}}}{\norm{ \boldsymbol{y}}},
\ee
where $\kappa_{\bA}$ is the condition number of $\bA$.

As mentioned in Section\,\ref{sec::system_model}, the Gaussian distribution of $n_i$'s is truncated in practice. This is used to provide a deterministic (non-probabilistic) guarantee on the accuracy of the computation result expressed in the following theorem.

\begin{theorem}
\label{thm-acc}
Let $\Delta f(s)$ denote the perturbation of $f(s)$ in the protocol discussed in Section \ref{sec::system_model} and $a_D$ denote the leading coefficient of $f(x)$. Let $r\leq m$. Then, 
\be\label{error_bound}
\Delta f(s)\leq a_D\sqrt{t+1}m^D  \kappa_{\bA}2^{-(v+1)},
\ee
where $t$ is the maximum number of colluding parties, $m$ is the truncation parameter of the Gaussian distribution, $\kappa_{\bA}$ is the condition number of $\bA$ given in \eq{A-def}, $v$ is the number of precision bits, and $r$ is the bound on the absolute value of the secret. In particular, by setting $\omega_i=\exp{(\frac{2\pi j i}{N})}$ for $i \in [N]$, we have
$$
\Delta f(s)\leq a_D\sqrt{t+1}m^D  2^{-(v+1)}.
$$
\end{theorem}

\begin{proof}
In order to recover $f(s)$, the  system of equations $f(p(\omega_i))=y_i$, for $i\in [N]$, is solved once all $y_i$'s are returned. This can be considered as a system of linear equations $\bA \bx=\by$, as described in \eq{A-def}. Observe that $\Delta s \leq\norm{\Delta \bx}$ and $\norm{\bx} \leq a_D\sqrt{t+1}m^D$. Hence,  \eqref{relative_error_ineq} implies that
\be\label{delta_x_bound}
 \frac{\Delta f(s)}{a_D\sqrt{t+1}m^D }\leq \frac{\norm{\Delta \bx}}{\norm{\bx}}.
 \ee
Moreover, since $v$ is the number of precision bits, we have
 \be\label{delta_y_bound}
 \frac{\norm{\Delta \boldsymbol{y}}}{\norm{ \boldsymbol{y}}} \leq 2^{-(v+1)}.
 \ee
Combining \eqref{delta_x_bound} with \eqref{delta_y_bound} yields
\be\label{error_bound}
\Delta f(s)\leq a_D\sqrt{t+1}m^D  \kappa_{\bA}2^{-(v+1)}.
\ee
In particular, if the evaluation points are the $N$-th roots of unity,  i.e., $\omega_i=\exp{(\frac{2\pi j i}{N})}$, the matrix $\bA$ turns into a unitary matrix for which $\kappa_\bA=1$. Hence, 
\be\label{error_bound_unity_roots}
\Delta f(s)\leq a_D\sqrt{t+1}m^D  2^{-(v+1)}.
\ee
\end{proof}

\noindent
\textbf{Remark 3.} Roughly speaking, picking the evaluation points as the roots of unity in the complex plane is \emph{optimal} from the accuracy point of view. This is because this particular choice results in the minimum possible condition number $\kappa_{\bA}$ and, consequently, minimizes the bound on the computation error provided in \Tref{thm-acc}. Throughout the rest of this paper we assume that the evaluation points are the $N$-th roots of unity. 

\section{Privacy Analysis}\label{sec:privacy}

In this section, we provide an analysis for the privacy level of data/secret in the proposed distributed computing protocol by considering two well-known notions of security, namely, mutual information security (MIS) and distinguishing security (DS). More specifically, we first consider these metrics of security assuming $t=1$, i.e., there is no collusion between the computation parties, in Section \ref{sub_sec::single_security}. Characterizing the privacy in the presence of $t$ colluding parties when $t>1$ is studied in Section \ref{sub_sec::coll_sec}. In these two sections it is assumed that the Gaussian distribution of noise terms $n_j$'s is not truncated, i.e., $m = \infty$. Then, in Section \ref{sub_sec::trunc_sec}, the results on the privacy of data are extended to cases with truncated Gaussian distribution for the noise terms. 

\subsection{Privacy against a single party}\label{sub_sec::single_security}

Consider the computational party $i$ for $i \in \{1,2,\dots,N\}$. Then the amount of information revealed to party $i$ about the secret $s$ can be measured in terms of the MIS metric, denoted by $\eta_c$, and defined as 
\be\label{sec-par}
\eta_c\,\deff\,\max_{i} \max_{P_S:|S|<r} I(S;Y_i),
\ee
where $P_S$ is the probability density function (PDF) of $S$. DS metric, denoted by $\eta_s$, is another metric for security which is defined using the \emph{total variation} (TV) distance metric $D_{\text{TV}}(.,.)$. In general, for any two probability measures $P_1$ and $P_2$ on a $\sigma$-algebra $\cF$, $D_{\text{TV}}(P_1,P_2)$ is defined as $\sup_{A \in \cF} |P_1(A) - P_2(A)| $. While DS metric is often defined for discrete random variables in the cryptography literature, it can be extended to real-valued random variables as follows:
\be
\label{semantic_security}
\eta_s\deff \max_i \max_{s_1,s_2 \in \D_S}D_{\text{TV}}(P_{Y_i|S=s_1},P_{Y_i|S=s_2}),
\ee
where $\D_S$ is the support of $S$. Note that both metrics $\eta_c$ and $\eta_s$ are non-negative. Also, roughly speaking, the smaller these metrics are the more private the secret $s$ is.


Upper bounding the security metric $\eta_c$ is discussed next. 
Since $|S|<r$, as discussed in Section\,\ref{sec::system_model}, we have $E[S^2]\leq r^2$. This together with \eq{sec-par} imply that
\be
\begin{split}
\label{upper_bound}
\eta_c = \max_{P_S:|S|<r} I(S;Y_i) &\leq \max_{P_S:E[S^2]\leq r^2} I(S;Y_i)\\ &=\log_2 (1+\frac{r^2}{\sigma_n^2}),
\end{split}
\ee
where the last equality is by the well-known result on the capacity of the additive white Gaussian noise (AWGN) channel \cite{cover2012elements}. Since the noise variance $\sigma_n^2$ can be picked arbitrarily large, one can assume $r = o(\sigma_n)$ to simplify the inequality in \eq{upper_bound} as follows:
\be\label{asym_etac}
\eta_c\leq \frac{1}{\ln 2} \frac{r^2}{\sigma_n^2}+o\bigl((\frac{r}{\sigma_n}\bigr)^2).
\ee

The notion of Hellinger distance, denoted by $H(.,.)$, is useful to bound the DS metric. It is defined as follows:
\be
H(P_1,P_2)\,\deff\,\frac{1}{2}\int(\sqrt{P_1}-\sqrt{P_2})^2d\psi .
\ee
The Hellinger distance can be bounded in terms of the total variation distance as follows \cite{kraft1955some}:
\be\label{Hellinger_statistical distance relation}
H(P_1,P_2)^2\leq D_{\text{TV}}(P_1,P_2) \leq \sqrt{2} H(P_1,P_2).
\ee
Let $P_1$ and $P_2$ be the PDFs of two complex Gaussian distributions both with variance $\sigma^2$and means $\mu_1$ and $\mu_2$, respectively. Also assume that the real and imaginary parts are independent and have identical variances.   Then we have \cite{pardo2018statistical} 
\be\label{Hellinger_distance_for_Gaussian}
H(P_1,P_2) = \sqrt{1-\exp({{-\frac{(\mu_1-\mu_2)^2}{4\sigma^2}})}}.
\ee
Using the aforementioned relations, the privacy parameter $\eta_s$ is bounded in the following theorem.

\begin{theorem}
\label{thm-semantic}
The DS metric $\eta_s$ is bounded as follows:
$$
\eta_s \leq \sqrt{2(1-\exp({{-\frac{r^2}{\sigma_n^2}}))}},
$$  
where $r$ is the maximum absolute value of the secret $s$ and $\sigma_n^2$ is the variance of the noise used in the proposed protocol. In particular, when $r = o(\sigma_n)$ we have 
$$
\eta_s \leq \sqrt{2}\frac{r}{\sigma_n}+o(\frac{r}{\sigma_n}).
$$
\end{theorem}

\begin{proof}
Note that the conditional distribution of $Y_i$ given $S = s_i$, specified in \eqref{semantic_security}, is $\cN(-s_i,\sigma_n^2)$. Then by using  \eqref{semantic_security} together with \eqref{Hellinger_statistical distance relation} and \eq{Hellinger_distance_for_Gaussian} we have 
\be
\label{semantic_security_upperbound}
\eta_s \leq \sqrt{2(1-\exp({{-\frac{r^2}{\sigma_n^2}}))}}.
\ee  
In particular, for $r = o(\sigma_n)$, \eqref{semantic_security_upperbound} is simplified to 
\be
\label{approx_semantic_security}\eta_s \leq \sqrt{2}\frac{r}{\sigma_n}+o(\frac{r}{\sigma_n}).
\ee
\end{proof}
Next, we discuss the relation between the two considered security metrics in the analog domain. It is known that the MIS and DS metrics can be directly related to each other over the space of discrete random variables \cite{bellare2012cryptographic}. In particular, it is shown that \cite{bellare2012cryptographic}:
\be \label{mis_ds_relation}
    \eta_s \leq \sqrt{2\eta_c},
\ee
assuming all the underlying random variables are discrete. We show in the next lemma that this result can be extended to the analog domain.

\begin{lemma}
\label{lem-semantic}
The inequality in \eq{mis_ds_relation} also holds when the underlying random variables, i.e., the secret as well as observations by parties, are continuous random variables.  
\end{lemma}
\begin{proof}
Let $X$ and $Y$ denote two continuous random variables and $X^\Delta$ and $Y^\Delta$ denote their quantized versions, respectively. Then we have \cite{cover2012elements} $$I(X,Y)= \lim_{\Delta \rightarrow 0} I(X^\Delta;Y^\Delta).$$ 
It can be observed that the same is true for the total variation distance, i.e., $D_{\text{TV}}(X,Y)= \lim_{\Delta \rightarrow 0} D_{\text{TV}}(X^\Delta;Y^\Delta)$. Hence, \eqref{mis_ds_relation} still holds assuming all the underlying random variables are continuous. 
\end{proof}

\Lref{lem-semantic} is used to bound $\eta_s$ later in Section \ref{sub_sec::coll_sec}. Note that one could apply it to derive a bound on $\eta_s$ using the bound on $\eta_c$ in \eq{upper_bound}. However, the resulting bound would be weaker comparing to the result stated in \Tref{thm-semantic}. 

Note that the amount of information revealed to a computational party, in terms of either of the security metrics, is a decreasing function of  $\frac{r}{\sigma_n}$. Furthermore, these metrics approach zero, i.e., the case with the \textit{perfect} privacy ($\eta_n,\eta_c=0$), as $\sigma_n \xrightarrow{} \infty$. Hence, increasing the noise variance improves the privacy of the scheme, However, this comes at the expense of reducing the precision of the result. This motivates studying the trade-off between the security metrics, as measures of data privacy, and the precision of the computations given a fixed number of bits to represent the floating-point numbers. This is the focus of Section\,\ref{sub_sec::trunc_sec}. In the next section, the results of this section are extended to the case with colluding parties. 

\subsection{Privacy against colluding parties}\label{sub_sec::coll_sec}

Let $t$ denote the number of colluding parties. The aim is to ensure the privacy of data against any subset of $t$ colluding computational parties. To this end, an upper bound on the amount of information revealed about the data/secret to the colluding parties is derived. Let $A=\{i_1,\hdots,i_t\}$ denote the set of indices for the colluding parties. Then the MIS metric is the mutual information between $S$ and all shares $Y_i$'s for $i\in A$, in the worst case, i.e., 
\be\label{leakage_adversaries}
\eta_c=\max_{A}I(S;Y_{i_1},\hdots,Y_{i_t}).
\ee 
Next it is shown that this can be upper bounded using the known results on the capacity of a single-input multiple-output (SIMO) channel under power constraints \cite{tse2005fundamentals}, similar to how the upper bound in \eqref{upper_bound} is obtained using the capacity result of AWGN channel. Let $\boldsymbol{h}$ and $\bN$ denote the channel coefficient vector and noise correlation matrix of a SIMO channel with $t$ output antennas, respectively. Then the capacity is given by \cite{tse2005fundamentals}
\be\label{SIMO_capacity}
C=\log_2 (1+p \norm{\boldsymbol{h}}^2\nu), 
\ee
where $p$ is the power of transmitted signal and $\nu$ is the maximum eigenvalue of $\bN^{-1}$. Consider a SIMO channel with  $\boldsymbol{h}_{t \times 1}=\mathbb{1} \deff(1,\hdots,1)^T$ and the correlated noise terms of $\tilde{n}_i \deff \sum_{j=1}^{t}\omega_i^jn_j$. Then the secret $s$ is mapped to the input of this channel. It can be observed that the shares given to $t$ servers can be mapped to the received symbols in this SIMO channel. Then the average input power is bounded by $r^2$, where $r$ is the maximum absolute value of $s$. Consequently, the capacity of the aforementioned SIMO channel with input power $r^2$ is an upper bound on the amount of information revealed to the $t$ colluding parties. Note that the coefficients $\omega_i$'s are the $N$-th roots of the unity, as discussed in Section \ref{sec::accuracy}. Then it can be observed that  $E[\tilde{n_j}\tilde{n_k}^*]=-\frac{\sigma_n^2}{t}$, for $j\neq k$, and $E[\tilde{n_j}\tilde{n_j}^*]=\sigma_n^2 $. Then, similar to \eqref{upper_bound}, one can write 
\be \label{leakage_to_adversaries}
\eta_c \leq \log_2 (1+\frac{r^2}{\sigma_n^2}  t\tilde{\nu}),
\ee
where $\tilde{\nu}$ is the maximum eigenvalue of $\bold{\tilde{N}}^{-1}$, where
$
\tilde{\bN}= \frac{t+1}{t} \bI_{t \times t} - \frac{1}{t} \mathbb{1}\mathbb{1}^t.
$
Note that $\tilde{N}$ has $t-1$ eigenvalues equal to $\frac{t+1}{t}$ and the last one is equal to $\frac{1}{t}$. This implies that $\tilde{\nu}=t$. Substituting this in \eqref{leakage_to_adversaries} yields:
\be \label{leakage_to_adversaries_final}
\eta_c \leq \log_2 (1+\frac{r^2t^2}{\sigma_n^2}),
\ee
providing an upper bound on the amount of information revealed to $t$ colluding parties in terms the MIS metric $\eta_c$. In particular, for $r= o(\sigma_n)$ we have 
\be\label{t_security_asympt}
\eta_c \leq \frac{t^2}{\ln 2} \frac{r^2}{\sigma_n^2}+o(\frac{r^2}{\sigma_n^2}).
\ee
Note that \eqref{t_security_asympt} is reduced to \eqref{asym_etac} for $t=1$.

Let $\eta_s$ denote DS metric for this case. By \Lref{lem-semantic} together with \eqref{leakage_to_adversaries_final} we have
\be\label{eta_s_tbound}
\eta_s \leq \sqrt{2\log_2 (1+t^2\frac{r^2}{\sigma_n^2})}.
\ee
In particular, for $r= o(\sigma_n)$ 
\be\label{etas_asympt}
\eta_s \leq \sqrt{\frac{2}{\ln 2}}t\frac{r}{\sigma_n}+o(\frac{r}{\sigma_n}).
\ee

\subsection{Privacy results with truncated noise}\label{sub_sec::trunc_sec}

The results provided on the security metrics so far are derived by assuming the additive noise terms $n_j$'s are drawn from a Gaussian distribution. While this assumption is valid in theory, such terms need to be truncated in practice as they can not be arbitrarily large. Furthermore, as shown in Section \ref{sec::accuracy}, in order to provide guarantees on the accuracy of the computations, $n_j$'s need to be bounded, i.e., $|n_j|\leq m$ for some $m \in \R^+$. In this section, we extend the results on bounding the security metrics in the proposed protocol to the case where $n_j$'s are drawn from a truncated Gaussian probability distribution. To simplify the computation, it is assumed that the truncation threshold is $m = \alpha \frac{\sigma_n}{\sqrt{t}}$, where $\alpha\in \R^+$ and $\frac{\sigma_n}{\sqrt{t}}$ is the standard deviation of the Gaussian distribution. 

First, the effect of truncation on the total variation distance metric is analyzed in the general case with $t$ colluding parties. In particular, we show that the change in the DS metric is exponentially small in terms of $\alpha$. Let $\eta_s'$ denote the DS metric after truncation of noise terms. Let $\Omega$ denote a $t$-dimensional complex vector space associated with $(y_{i_1},\hdots,y_{i_t})$, where $y_{i_j}$'s are defined in \eqref{system}. Let $P_{\bY_t}$ and $Q_{\bY_t}$ denote the PDFs of $\bY_t\deff(Y_{i_1},\hdots,Y_{i_t})$ given $s=r$ and $s=-r$, respectively, when the noise terms are not truncated. Similarly, $\tilde{P}_{\bY_t}$ and $\tilde{Q}_{\bY_t}$ are defined when the noise terms are truncated. Also, Let $B_1=\{\by_t\in\Omega:\tilde{P}_{\bY_t}(\by_t)\neq0\}$, $B_2=\{\by_t\in\Omega:\tilde{Q}_{\bY_t}(\by_t)\neq0\}$ and $B_{12}=B_1\cap B_2$.  

Note that $\tilde{P}_{\bY_t}(\by_t)=\frac{1}{w}P_{\bY_t}(\by_t)$ and $\tilde{Q}_{\bY_t}(\by_t)=\frac{1}{w}Q_{\bY_t}(\by_t)$, for $\by_t \in B_1$ and $\by_t \in B_2$, respectively, and are zero otherwise, where $w$ is given by
\begin{align}
 w&=\Pr[(|n_1|<m,\hdots,|n_{t}|<m)]\\
 &=\Pr[|n_1|<m]^{t}\geq (1-2\exp(-\frac{\alpha^2}{2}))^{t}\label{gaussian_tail},
\end{align}
where the inequality is by bounding the tail distribution function of the standard normal distribution. One can observe that the TV distance in \eqref{semantic_security} is maximized when $s_1=r$ and $s_2=-r$.   
Then, using an alternative definition of the total variation distance when the probability measures are over $\R$ we can write:
\begin{align}
\label{etas-new}
\eta_s'&=\frac{1}{2}\int_\Omega|\tilde{P}_{\bY_t}(\by_t) - \tilde{Q}_{\bY_t}(\by_t)| d\by_t\\  
&=\frac{1}{2}\int_{B_{12}} |\tilde{P}_{\bY_t}(\by_t) - \tilde{Q}_{\bY_t}(\by_t)| d\by_t \label{integral_seperation_first}\\
&+\frac{1}{2}\int_{B_{12}^c} |\tilde{P}_{\bY_t}(\by_t) - \tilde{Q}_{\bY_t}(\by_t)| d\by_t.
\label{integral_seperation_second}
\end{align}
The term in \eqref{integral_seperation_first} is bounded as follows:
\begin{align}
 &\int_{B_{12}} \hspace{-2mm}|\tilde{P}_{\bY_t}(\by_t) - \tilde{Q}_{\bY_t}(\by_t)| d\by_t=\frac{1}{w}\int_{B_{12}} \hspace{-2mm} |{P}_{\bY_t}(\by_t) - {Q}_{\bY_t}(\by_t)| d\by_t \\
 &\leq \frac{1}{w}\int_{\Omega} |{P}_{\bY_t}(\by_t) - {Q}_{\bY_t}(\by_t)| d\by_t=\frac{1}{w}\eta_s,\label{domain_exp}
\end{align}
where \eqref{domain_exp} is by noting that $B_{12}\subset \Omega$. In order to derive an upper bound on the term in \eqref{integral_seperation_second} note that
\begin{align}
  &\int_{B_{12}^c} |\tilde{P}_{\bY_t}(\by_t) - \tilde{Q}_{\bY_t}(\by_t)| d\by_t \\
  & \leq \int_{B_{12}^c} \tilde{P}_{\bY_t}(\by_t) d\by_t +  \int_{B_{12}^c} \tilde{Q}_{\bY_t}(\by_t) d\by_t\\
  &=2 \int_{B_{12}^c} \tilde{P}_{\bY_t}(\by_t) d\by_t\label{region_symmetry},
\end{align}
where \eqref{region_symmetry} is due to symmetry. An upper bound on the term in \eqref{region_symmetry} is derived in the following lemma.   

\begin{lemma}\label{second_integral_bound}
We have
$$\int_{B_{12}^c} \tilde{P}_{\bY_t}(\by_t)\leq (2\exp(-\frac{1}{2}(\alpha-\frac{2r\sqrt{t}}{\sigma_n})^2))^t. $$
\end{lemma}
\begin{proof}
Let $P'_{\bY_t}$ denote the PDF of the random vector $\bY_t$ given $s=r$ and assuming $n_i$'s are drawn from a truncated Gaussian distribution with threshold $m-2r$. Similar to the definition of $B_1$, let $B'_1\deff\{\by_t\in \Omega: p'_{Y_t}(\by_t)\neq0\}$. Since the equations relating $y_i$'s and $n_i$'s in \eqref{system} are linear, then it can be observed that $B'_1\subset B_{12}$. Then we have 
\begin{align}
&\int_{B_{12}^c} \tilde{P}_{\bY_t}(\by_t) d\by_t\leq\int_{B_{12}^c}\frac{1}{w} P_{\bY_t}(\by_t) d\by_t\label{pdf_inequality}\\
&\leq\int_{{B'_1}^c} P_{\bY_t}(\by_t) d\by_t \label{domain_change}\leq (2\exp(-\frac{1}{2}(\alpha-\frac{2r\sqrt{t}}{\sigma_n})^2))^t
\end{align}
where \eqref{pdf_inequality} holds  because $\tilde{P}_{\bY_t}(\by_t)$ is either equal to $\frac{1}{w}P_{\bY_t}(\by_t)$ or zero, the first inequality in \eqref{domain_change} holds since  $B'_1\subset B_{12}$ implies $B_{12}^c \subset {B'_1}^c$, and the second one is by bounding the tail distribution function of the standard normal distribution
\end{proof}

\begin{theorem}\label{ds_security_thm}
The DS metric for the case where $n_j$'s are drawn from a truncated Gaussian distribution with truncation level $\alpha \frac{\sigma_n}{\sqrt{t}}$  satisfies the following inequality:
$$
\eta'_s \leq \frac{1}{w} \eta_s+\frac{1}{w}(2\exp(-\frac{1}{2}(\alpha-\frac{2r\sqrt{t}}{\sigma_n})^2))^t\label{gaussian_tail_bound_2},
$$
where $w\geq (1-2\exp(-\frac{\alpha^2}{2}))^t$.
\end{theorem}
\begin{proof}
The proof follows by \eq{etas-new} together with bounding \eqref{integral_seperation_first} using \eqref{domain_exp} and \eqref{integral_seperation_second} using the result of Lemma  \ref{second_integral_bound}, respectively. 
\end{proof}

Theorem \ref{ds_security_thm} implies that picking, for instance, $\alpha = 10 $ with $t=10$, and already having a very small $\frac{r}{\sigma_n}$ is sufficient to obtain almost the same bound on the DS metric as in the case where the noise terms are not truncated. Hence, truncation of the noise terms in \eqref{system} does not compromise the privacy of data in the protocol as long as $\alpha$ is picked sufficiently large. 

In order to obtain a similar result for the MIS metric, a result on the capacity of channels with additive truncated Gaussian noise is needed. This problem is studied recently, see, e.g., \cite{egan2017capacity}. In particular, it is shown that the capacity of AWGN channel is \emph{robust} against truncation of the noise. More specifically, it is shown that the change in the capacity by truncating the noise is $O(\exp (-\frac{\alpha^2}{2}))$ \cite{egan2017capacity}. Hence, the MIS metric is increased by at most $O(\exp (-\frac{\alpha^2}{2}))$ when truncating the noise, mimicking the result derived for the DS metric in \Tref{ds_security_thm}


In Table \ref{table::trade-off}, the  trade-off between the privacy and the accuracy of our protocol is demonstrated using the theoretical results obtained in Section \ref{sec::accuracy} and Section \ref{sec:privacy}. It can be observed that increasing the variance of the noise $\sigma_n$ improves the privacy but at the same time reduced the accuracy of the computations.

\begin{table}[]
    \centering
\begin{tabular}{ |c|c|c|c|c| } 
\hline
$\log_{10}(\sigma_n)$ & $5$ & $11$& $18$ \\
\hline
 $\log_{10}(\Delta f(s))$ & $-9.80$ & $-4.80$ & $0.196$\\ 
 \hline
$\log_{10}(\eta_s)$& $-2.36$ & $-7.35$&$-12.4$ \\ 
\hline
\end{tabular}
    \caption{\small Demonstration of the trade-off between DS security metric and accuracy. The upper bound on $\eta_s$ in \eqref{leakage_to_adversaries_final} is calculted versus the upper bound on  $\Delta f(s)$ obtained in \Tref{thm-acc} by at $\sigma_n = 10^5, 10^{10}, 10^{18}$. Other parameters are $a_D=1$, $t=1$, $D=1$, $\alpha=10$, $r=255$ and $v=52$.}
    \label{table::trade-off}
\end{table}

\section{Experiments}\label{sec:experiments}

In this section, we demonstrate experiment results on the performance of our proposed protocol when applied to a certain learning algorithm. First, it is shown that the accuracy of the results obtained by using our protocol in a distributed setting closely follows that of a conventional centralized approach, thereby providing almost the same accuracy as in the centralized approach. Second, the performance of our protocol is compared with that of the state-of-the-art schemes employing fixed-point numbers by quantizing the data and mapping it to finite field elements. In particular, we compare our protocol with CodedPrivateML \cite{so2019codedprivateml} in terms of accuracy and run time.

The problem of training a logistic regression (LR) model over MNIST dataset is considered. Let $\bX \in \R^{m \times d}$ denote a dataset consisting of $m$ samples with $d $ features and $\bl\in \{0,1\}^m$ denote the corresponding label vector. The task is to compute the model parameters (weights) $\bw\in \R^d$ by iteratively minimizing the cross entropy function using the following parameter update equation:
\be \label{updates}
\bw^{(j+1)}=\bw^{(j)}-\frac{\beta}{m}\bX^T (g(\bX\bw^{(j)})-\bl)),
\ee
where $\bw$ is the estimated parameters in iteration $i$, $\beta$ is the learning rate, and $g(x)\deff \frac{1}{1+\exp(-x)}$ is the sigmoid function that operates element-wise over the vector inputs. For each data point $\bx_i\in\R^{1\times d}$, the estimated probability of $l_i$ being equal to $1$ is $g(\bx_i\bw)$. All experiments are performed in MATLAB and the considered problem is the binary classification between digits $3$ and $7$ over MNIST dataset. Our protocol for the distributed training of the LR model is inspired by  Algorithm\,\ref{scheme} and is described in Algorithm\,\ref{scheme2}. This protocol is implemented using the default double-precision floating-point (FLP) representation in MATLAB with $64$ bits, where $v = 52$, $q=11$, and the other bit is reserved for the sign. 

\begin{algorithm}[H]
\SetAlgoLined
 \textbf{Input:} Dataset $\bX \in \R^{m\times d}$, the number of iterations $k$ and $\alpha$.\\
 \textbf{Public parameters:} $(\omega_1, \hdots, \omega_N)$, $\tilde{\mathbf{b}}_{1 \times N}=(\tilde{b}_1, \hdots, \tilde{b}_N)$.   \\
\textbf{Output:} Parameter vector $\bw$ for the logistic regression model.

\textbf{Encoding dataset (at the master):}

Pick i.i.d.  $\bN_j \in \R^{m\times d} $ with entries independently drawn from $ \cN(0,\frac{\sigma_n^2}{t})$ truncated at $\alpha \frac{\sigma_n}{\sqrt{t}}$, for $j=1,\dots,t$.


\For {$i \in [N]$}
{Compute $\tilde{\bX}_i=\bX+\sum_{j=1}^{t} \omega_i^j\bN_j$.}
Send $\tilde{\bX}_i$ to server $i$.

\textbf{Computation of $\bw$ iteratively:}  

Set $\bw^{(0)}=\boldsymbol{0}.$

\For {$j \in \{0, \hdots, k-1\}$}
{
\textbf{Encoding phase (at the master):}

Pick i.i.d.  $\mathbf{n}_j \in \R^{1\times d} $ with entries independently drawn from $ \cN(0,\frac{\sigma_n^2}{t})$ truncated at $\alpha \frac{\sigma_n}{\sqrt{t}}$, for $j=1,\dots,t$.

 Compute $\tilde{\bw}_i^{(j)}=\bw^{(j)}+\sum_{h=1}^{t} \omega_i^h\mathbf{n}_h$.
  
  Send $\tilde{\bw}_i^{(j)}$ to server $i$.
  
  \textbf{Computation phase (at server $i$):}
  
  Compute $\bz_i= \tilde{\bX}_i^T\tilde{\bX}_i\tilde{\bw}_i^{(j)}$.
  
  Send $\bz_i$ to the master node.
  
  \textbf{Decoding phase (at the master):}
  
  Compute $\bu^{(j)}=\sum_{i=1}^{N} \tilde{b}_i{\bz}_i$.
  
  Update $\bw^{(j+1)}=\bw^{(j)}-\frac{\beta}{2m} [\frac{1}{2}\bu^{(j)}+\bX^T(\mathbb{1}-2\bl)]$. 
  }
Return $\bw=\bw^{(k)}$.

\caption{Privacy-preserving distributed training of logistic regression model in the analog domain. }
\label{scheme2}
\end{algorithm}

Next, we describe the steps in Algorithm\,\ref{scheme2} in details. In the beginning, the data matrix $\bX$ is encoded element-wise using the analog counterpart of Shamir's encoder, same as in \eqref{system}, and then the secret shares are sent to the servers. Let $\tilde{\bX}_i$ denote the share sent to server $i$, for $i \in [N]$. The initial parameter vector is set to the all-zero vector, i.e., $\bw^{(0)}=\mathbf{0}$. Let $k$ denote the total number of iterations for updating the model parameters using \eqref{updates} in the experiment. In the $j$-th iteration, for $j \in \{0, \hdots, k-1\}$, the master node encodes $\bw^{(j)}$ element-wise, again same as in \eqref{system}, and sends the shares to the servers. Let  $\tilde{\bw}_i^{(j)}$ denote the share of  $\bw^{(j)}$ sent to server $i$. The server $i$ then computes $ \tilde{\bX}_i^T\tilde{\bX}_i\tilde{\bw}_i^{(j)}$ and returns the result to the master node. Next, the master node recovers $\bX_i^T\bX_i\bw_i^{(j)}$ by computing a linear combination of the returned results, same as in the decoding phase in Algorithm\,\ref{scheme}, and utilizes it to update the vector of parameters according to \eqref{updates} with the sigmoid function substituted by its $1$-degree polynomial approximation, i.e., $g(x)\approx \frac{1}{2}+\frac{x}{4}$. This procedure is continued till the desired number of iterations is passed and the last update of the parameter vector is returned as the final result of the protocol. It is worth mentioning that the data matrix $\bX$ is secret-shared only once at the beginning and the same shares are used at each iteration by the servers while the parameter vector $\bw$ is updated and secret-shared in each iteration.

The vector of model parameters $\bw$ for the training dataset is computed using Algorithm\,\ref{scheme2} as well as using the conventional centralized method. The number of servers $N=4$ and $t=1$ are assumed. Note that in the centralized method the sigmoid function is not approximated while in our implementation it is approximated with a degree-1 polynomial. Then the accuracy of the predictions are determined over the MNIST test dataset in both approaches. The result is shown in Figure \ref{accuracy}. It can be observed that the accuracy of our protocol closely follows that of the conventional centralized approach.

\begin{figure}[t]
	\begin{center}
		\includegraphics[width=0.9\linewidth]{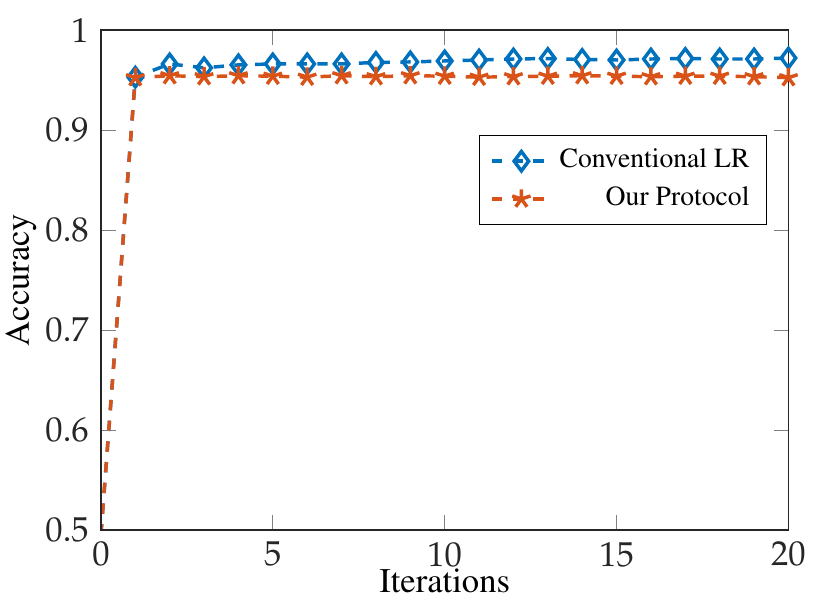}
		\caption{\small Comparison between the accuracy of our distributed learning protocol and the conventional centralized logistic regression (LR). }\label{accuracy}
	\end{center}
\end{figure}

It this setting with honest-but-curious servers, as mentioned in Section\,\ref{sec:one-A}, the servers may attempt to infer the data by accumulating all received shares during all iterations. Since $\mathbf{n}_h$'s in Algorithm\,\ref{scheme2} are picked independently in each iteration, the leakage of information for the model in terms of DS metric is bounded by $k \eta_s$, where $k$ is the number of iterations and $\eta_s$ is characterized in \eqref{approx_semantic_security} for $t=1$ and in \eqref{eta_s_tbound} for $t>1$. Furthermore, the privacy guarantee for the dataset in terms of the DS metric is given by \eqref{approx_semantic_security} for $t=1$  and by \eqref{eta_s_tbound} for $t>1$, regardless of the value of $k$ since the dataset is encoded and sent to the servers only once during the protocol. Hence, given all the parameters in the described experiment, the privacy guarantee in our protocol in terms of the DS metric is $\eta_s \leq 2\times10^{-14}$ for the model and $\eta_s \leq 10^{-15}$ for the dataset. These hold by utilizing \eqref{approx_semantic_security}, where $\sigma_n=10^{18}$ is picked, and setting $l\leq 20$ in all experiments while noting that the maximum absolute value of data is $r=255$ in MNIST dataset. 

In the second experiment, the accuracy of a fixed-point (FXP) implementation, according to the protocol proposed in CodedPrivateML \cite{so2019codedprivateml}, is simulated in a similar scenario with $N=4$ and $t=1$, and is compared with that of our protocol. All other parameter are picked according to what is reported in \cite{so2019codedprivateml}, which also uses $64$ bits to represent elements of the finite field. Figure \ref{floatVSfixed} demonstrates that the accuracy of CodedPrivateML (fixed point) is significantly dropped to around $0.5$, equivalent to that of a random guessing, when the size of dataset exceeds $100$. Note that the original train and test datasets consist of $12396$ and $2038$ samples, respectively. In order to observe the performance with small dataset sizes, we pick a dataset with equal data points labeled with $3$ and $7$ in each experiment. Also, we run the experiment $1000$ times by picking different sets of samples and the average accuracy is reported in Figure \ref{floatVSfixed}  .

\begin{figure}[t]
	\begin{center}
		\includegraphics[width=0.9\linewidth]{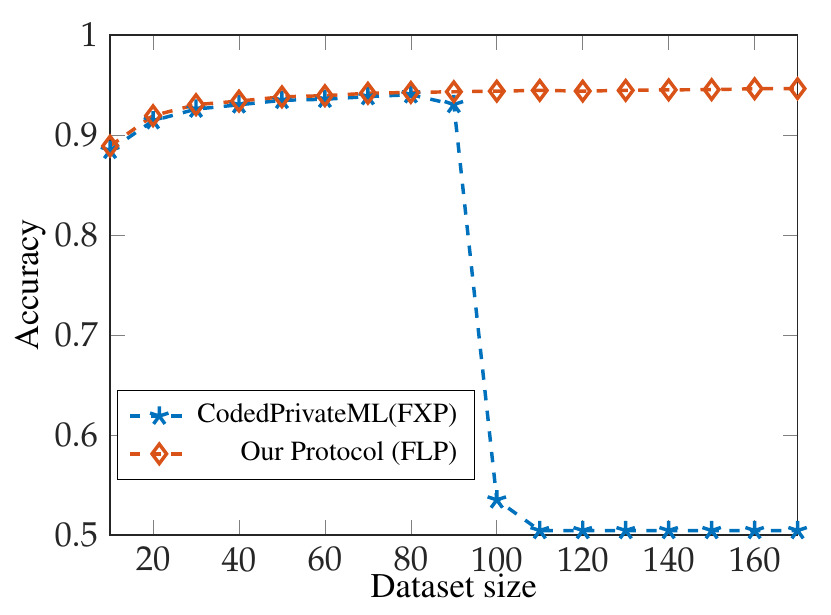}
		\caption{\small Comparison between the accuracy of our protocol and CodedPrivateML \cite{so2019codedprivateml} implementations. The number of iterations in both cases is $15$.
		}\label{floatVSfixed}
	\end{center}
\end{figure}


This comparison demonstrates the superiority of our proposed protocol in the analog domain and implemented using floating point numbers comparing to the state-of-the-art distributed computing and learning schemes employing quantization followed by computations over a finite field. In other words, our protocol is \emph{robust} with respect to the size of the training dataset while the fixed-point implementations suffer significantly from wrap-around error as the size of dataset passes a certain threshold depending on the prime number picked as the size of underlying finite field.

\begin{table}[]
    \centering
\begin{tabular}{ |c|c|c|c| } 
\hline
Dataset size & CodedPrivateML & Our Protocol \\
\hline
$1000$ & $0.72$ & $0.25$ \\ 
$2000$& $1.49$ & $0.52$ \\ 
$3000$& $2.49$ & $0.80$ \\ 
$4000$& $3.87$ &$1.09$\\
$5000$& $5.94$ &$1.38$\\
\hline

\end{tabular}
    \caption{\small Comparison of the run times between the fixed-point and the floating-point implementations. The times are reported in seconds. The experiments are done on a Macbook pro with 3.5 GHz dual-core intel core i7 CPU and 16 GB memory.}
    \label{table}
\end{table}

One major advantage of CodedPrivateML over MPC-based approaches is that it provides an order of magnitude speed up, based on the experiment results reported in \cite{so2019codedprivateml}. The reason is that in CodedPrivateML, there is no communication between computation parties thereby improving the communication complexity of the scheme significantly, compared with the state-of-the-art cryptographic approaches. This advantage is preserved in our protocol as well, since no communication is needed between the parties.  

Note that in order to avoid the wrap-around error in the fixed-point implementation each computation party should stop the computation before the the wrap-around threshold is passed and divide the computation task into smaller subtasks.  Then, it needs to send back all the computation results associated to each subtask to the master node in order to guarantee recovery of the computation result. This results in an excess communication and computation overhead compared with our protocol. Moreover, since the threshold is not known \emph{a priori}, one always needs to check if the wrap-around is occurred during the computation process. These factors slow down CodedPrivateML when the dataset is large and one wants to avoid the errors due to wrap-around. In Table \ref{table}, the computation times of CodedPrivateML and our protocol are compared for the experiment discussed in this section for different dataset sizes, while discarding the delay in CodedPrivateML due to frequently checking wrap-around errors and communication overhead.  
It shows that for the same level of accuracy of the results, our approach with the floating-point implementation also outperforms the fixed-point implementations while preserving the speed up advantage compared with the MPC-based schemes.

\section{Conclusion and Discussions}\label{sec:conclusion}
In this paper, we tackled the critical problem of privacy-preserving computation over a real-valued dataset using distributed honest-but-curious servers. To this end, we proposed a protocol that utilizes a counterpart of Shamir's secret sharing scheme in the analog domain. In order to measure the privacy level of the data, the conventional notion of distinguishing security is extended to the analog domain and privacy guaranties for the proposed scheme are characterized based on this security metric. The well-known connection between the DS and the MIS measures of security is extended from the discrete domain to the continues domain. This is then utilized to bound the DS metric of our protocol using well-known results on the capacity of SIMO channel with correlated noise. Furthermore, the accuracy of the outcome of the computation is characterized assuming a floating-point implementation of the protocol. In our experiments, we illustrated that the accuracy of the predictions for the logistic regression model over the MNIST dataset derived by our protocol closely follows that of the conventional centralized approach. Finally, we showed that our protocol is robust with respect to the size of the training dataset, i.e., there is almost no accuracy loss as the size of the training dataset grows large, while the performance of the fixed-point implementations in prior work significantly diminishes due to overflow errors.

There are several directions for future work. Extending the proposed protocol in this paper to scenarios with straggler servers is an interesting direction for future research. More specifically, in our protocol it is assumed that all the servers successfully finish their assigned tasks, while a certain number of servers, referred to as stragglers, may be slow or may not respond at all in practice \cite{lee2018speeding,li2016unified,yu2020straggler,reisizadeh2019coded,aliasgari2018coded, jamali2019coded}. The main challenge in this direction is to pick the parameters of the protocol and to design the decoder that is better than the naive and numerically unstable approach of solving a system of linear equations in the analog domain. Another direction is to adopt the proposed protocol in this paper to perform computational tasks in distributed fashion for other applications, such as distributed optimization and mechanism design \cite{heydariben2018distributed,rabbat2004distributed,heydaribeni2019distributed,zhang2018improving,heydaribeni2018distributed}, while keeping the data private. Generalizing Algorithm\,\ref{scheme} in order to simultaneously compute multiple evaluations of a polynomial in a single-shot is another future direction. To this end, techniques for multi-user secret sharing can be utilized \cite{soleymani2018distributed}. Obtaining such results can potentially lead to privacy-preserving multi-task learning protocols, i.e., protocols that train multiple models over a dataset in a single round.

\bibliographystyle{IEEEtran}
\bibliography{ref}

\end{document}